\pdfoutput=1
\def\SIAMCleanVersion{1}
\ifdefined\SIAMCleanVersion
  \documentclass[onefignum,onetabnum]{siamonline171218}
\else
  \documentclass[review,onefignum,onetabnum]{siamonline171218}
\fi

\usepackage{amsfonts}
\usepackage{amssymb}
\usepackage{mathtools}
\usepackage{booktabs}
\usepackage{array}
\usepackage{enumitem}
\usepackage{microtype}

\newsiamthm{assumption}{Assumption}
\newsiamremark{remark}{Remark}

\newcommand{\Hsp}{\mathsf H}
\newcommand{\Vsp}{\mathsf V}

\newcommand{\Hm}{\mathsf H_m}
\newcommand{\R}{\mathbb R}
\newcommand{\N}{\mathbb N}
\newcommand{\Pp}{\mathbb P}
\newcommand{\Ee}{\mathbb E}
\newcommand{\Law}{\operatorname{Law}}

\newcommand{\Lip}{\operatorname{Lip}}
\newcommand{\cP}{\mathcal P}
\newcommand{\cG}{\mathcal G}
\newcommand{\cF}{\mathcal F}
\newcommand{\cL}{\mathcal L}

\newcommand{\cA}{\mathcal A}

\newcommand{\cD}{\mathcal D}

\newcommand{\norm}[1]{\left\lVert #1\right\rVert}
\newcommand{\abs}[1]{\left\lvert #1\right\rvert}
\newcommand{\inner}[2]{\left\langle #1,#2\right\rangle}
\newcommand{\dd}{\,\mathrm d}
\newcommand{\dt}{\,\mathrm dt}
\newcommand{\Wtwo}{W_2}

\headers{Consistency of Functional Flow Matching}{L. J. Shikhman}
\title{Discretization and Statistical Consistency of Functional Flow Matching}
\author{Lennon J. Shikhman\thanks{College of Computing,
Georgia Institute of Technology, Atlanta, GA 30332, USA
(\email{lshikhman3@gatech.edu}).\newline
An archived preprint of this work is available as arXiv:2608.04531 (\url{https://arxiv.org/abs/2608.04531}).}}

\hypersetup{
  pdftitle={Discretization and Statistical Consistency of Functional Flow Matching},
  pdfauthor={Lennon J. Shikhman},
  pdfkeywords={functional flow matching, neural operators, Hilbert spaces, finite sensors, Wasserstein convergence}
}

\begin{document}
\maketitle

\begin{abstract}
Functional flow matching learns to generate random functions, such as velocity
fields or solutions of partial differential equations, by continuously
transforming a simple probability distribution into a target distribution.
In practice, however, training and sampling use only finitely many grid values
or coefficients. Does refining these observations recover the intended
function-space model, and how do observation, learning, and numerical errors
affect the generated distribution? We address these questions by first proving
that the velocity targets used for training converge in mean square when the
finite reconstructions converge to the underlying functions and remain
uniformly bounded. The observations need not be nested: each refinement may replace
earlier sensor locations. We also give quantitative bounds for orthogonal
projections and extend the result to point observations under suitable
regularity and reconstruction assumptions. We then bound the discrepancy
between the generated and target distributions in Wasserstein distance,
separating observation, learning, numerical, and endpoint errors without
requiring uniqueness of the population finite-dimensional dynamics. For a
normalized quadrature neural operator representing the learned velocity field,
we establish stability and
learning bounds with constants independent of the number of sensors. Under
fixed model dimension and boundedness assumptions, the statistical excess
risk decreases at an inverse-sample-size rate, up to logarithmic factors.
A Gaussian example combines these estimates into an explicit
sampling-error rate.
\end{abstract}



\section{Introduction}
\label{sec:introduction}

Functional flow matching generates random functions, such as velocity
fields or PDE solutions, by learning a time-dependent velocity on function
space \cite{kerrigan2024functional}. Training and sampling use finite
observations of those functions. When the grid or sensor locations change,
the information available to the model changes too. We study whether the
resulting training targets and generated distributions converge to their
function-space counterparts as the observations become more accurate.

Let $X_t$ be a stochastic interpolation with velocity $U_t$, and let $A_m$
reconstruct a function from finite observations. The continuum and finite
training targets are
\[
 v^\star(t,X_t)=\Ee[U_t\mid X_t],
 \qquad
 v_m^\star(t,A_mX_t)=\Ee[A_mU_t\mid A_mX_t].
\]
The second target averages velocities over all functions compatible with
the observed values. Changing the observations therefore changes the
conditional expectation, not just its output coordinates. In the Gaussian
example of Section~\ref{subsec:noncommuting-gaussian}, projecting the
continuum field gives boundary coefficient $0$, whereas conditioning on
the finite observations gives $0.72$.

Refining a reconstruction may replace earlier sensor locations.
The new observations need not include the old ones, so the usual
martingale argument does not apply. We prove strong $L^2$ convergence
for consistent finite-rank reconstructions and use it to analyze the
learned sampler.

The learned velocity maps a function and time to another function. A neural
operator can represent this map, while the flow-matching objective specifies
how it is trained. We analyze a normalized quadrature neural operator to
connect the function-space velocity to its finite-sensor implementation.

\subsection{Contributions}
\label{subsec:contributions}

\begin{enumerate}[leftmargin=2.1em,itemsep=2pt]
 \item \textbf{Convergence of training targets.}
 Uniformly bounded, strongly consistent finite-rank reconstructions give
 $v_m^\star(\tau,A_mX_\tau)\to v^\star(\tau,X_\tau)$ in $L^2$.
 We bound the error for orthogonal projections and extend the argument
 to point sensors using a regularity space.
 \item \textbf{Error in the generated distribution.}
 We compare the learned ODE with paths representing the population law.
 The resulting Wasserstein bound separates sensing, learning, numerical,
 and endpoint errors and requires stability of the learned field.
 \item \textbf{Bounds for a neural-operator velocity field.}
 For normalized quadrature operators, we derive constants independent of
 the sensor count and bound the error in approximating the continuum
 operator. The statistical excess risk has order $\widetilde O(n^{-1})$
 at fixed model dimension and envelopes. A clipped Gaussian model gives
 an explicit rate for the complete sampler.
\end{enumerate}

\section{Related Work}
\label{sec:related}

Flow matching \cite{lipman2023flow}, rectified flow
\cite{liu2023rectified}, and stochastic interpolants
\cite{albergo2023stochastic} learn conditional velocities along prescribed
probability paths. Functional flow matching \cite{kerrigan2024functional}
and Functional Mean Flow \cite{li2025functionalmean} extend these ideas to
function spaces. Recent theory establishes Hilbert-space marginal
preservation, conditional-drift regularity, and objective equivalence
\cite{zhang2026functional,chen2025scaleadaptive,li2026fotcfm}; operator flow
matching and infinite-dimensional probability-flow ODEs provide related
formulations \cite{shi2025operatorflow,na2025probability}. The discretization question concerns how these conditional targets change
with the observations. Function-space diffusion has related approximation
and multilevel analyses \cite{pidstrigach2024infinite,hagemann2025multilevel};
here the object to be approximated is itself a conditional expectation
that depends on the observation map. Point observations also require regularity and stable recovery
\cite{wendland2005scattered,narcowich2005sobolev}; evaluation on $H^s(D)$ is
continuous only for $s>d/2$ \cite{adams2003sobolev}.

Neural operators share parameters across discretizations
\cite{kovachki2023neural}, with FNO and DeepONet providing prominent
architectures and approximation theories
\cite{kovachki2021fno,lu2021deeponet,lanthaler2022error}. For a fixed set of parameters, changing the quadrature changes the
implemented operator. We bound this error for normalized quadrature
operators and then propagate it to the generated law.
Finite-dimensional flow-matching and neural-operator learning theory supply
complementary statistical tools
\cite{benton2024error,fukumizu2025minimax,reinhardt2024statistical}.
Superposition principles \cite{ambrosio2008gradient,stepanov2017superposition}
represent solutions of the continuity equation by distributions on
characteristic curves. This representation supplies the population paths
used in our comparison with the learned flow.

\section{Setting and notation}
\label{sec:setup}

Let $\Hsp$ be a separable real Hilbert space and $\cP_2(\Hsp)$ its Borel
laws with finite second moment, equipped with
\[
 \Wtwo^2(\mu,\nu)=\inf_{\pi\in\Pi(\mu,\nu)}
       \int\norm{x-y}_{\Hsp}^2\,\pi(\dd x,\dd y).
\]
Conditional expectations are Bochner conditional expectations; vector
fields are identified $\dd t\,\mu_t(\dd x)$-almost everywhere unless a
representative is specified.

\subsection{The stochastic interpolation}

Fix $T\in(0,1]$ and a complete probability space $(\Omega,\cA,\Pp)$.

\begin{assumption}[Absolutely continuous interpolation]
\label{ass:interpolation}
The map $X:[0,1]\times\Omega\to\Hsp$ is jointly measurable, almost every
path is absolutely continuous, and its jointly measurable Bochner derivative
$U_t=\dot X_t$ satisfies
\[
  \Ee\norm{X_0}_{\Hsp}^2<\infty,
  \qquad
  \int_0^1\Ee\norm{U_t}_{\Hsp}^2\dt<\infty.
\]
\end{assumption}

Then $X_t=X_0+\int_0^tU_r\dd r$ and
\begin{equation}
\label{eq:path-sup-moment}
  \Ee\sup_{0\le t\le T}\norm{X_t}_{\Hsp}^2
  \le
  2\Ee\norm{X_0}_{\Hsp}^2
  +
  2T\int_0^T\Ee\norm{U_t}_{\Hsp}^2\dt
  <\infty.
\end{equation}
Let $\mu_t=\Law(X_t)$ for $0\le t\le1$.  All training risks and flow
equations below are restricted to $[0,T]$, while $\mu_1$ denotes the desired
endpoint law.  A standard example is the linear interpolation
\begin{equation}
\label{eq:linear-interpolation}
  X_t=(1-t)X_0+tX_1,\qquad U_t=X_1-X_0,
\end{equation}
for square-integrable $X_0,X_1$.

Regard time as random: on $\overline\Omega=[0,T]\times\Omega$, let
$\overline\Pp=T^{-1}\mathcal L^1|_{[0,T]}\otimes\Pp$,
$\tau(t,\omega)=t$, and
$\cG=\sigma((t,\omega)\mapsto(t,X_t(\omega)))$.
The factorization lemma on the standard Borel space $[0,T]\times\Hsp$
\cite[Chapter~1]{kallenberg2021foundations} gives a Borel field $v^\star$ with
\begin{equation}
\label{eq:population-target}
 v^\star(\tau,X_\tau)=\Ee[U_\tau\mid\cG]
 \quad\overline\Pp\text{-almost surely}.
\end{equation}
We abbreviate this as $v^\star(t,x)=\Ee[U_t\mid X_t=x]$.

For a Borel field $w:[0,T]\times\Hsp\to\Hsp$ with
$w(\tau,X_\tau)\in L^2(\overline\Pp;\Hsp)$, define the normalized population
risk
\begin{equation}
\label{eq:population-risk}
  \cL(w)
  =
  \frac1T\int_0^T
  \Ee\norm{w(t,X_t)-U_t}_{\Hsp}^2\dt.
\end{equation}

\subsection{Finite-rank reconstructed observations}

Let $A_m\in\mathcal L(\Hsp)$ have finite-dimensional range $\Hm$ and satisfy
\begin{equation}
\label{eq:strong-reconstruction}
  A_mx\longrightarrow x
  \quad\text{in }\Hsp\quad\text{for every }x\in\Hsp.
\end{equation}
The uniform boundedness principle implies
\begin{equation}
\label{eq:uniform-reconstruction-bound}
  C_A:=\sup_{m\ge1}\norm{A_m}_{\mathcal L(\Hsp)}<\infty.
\end{equation}

Set $X_t^m=A_mX_t$, $U_t^m=A_mU_t$,
$\mu_t^m=(A_m)_\#\mu_t$, and
$\cG_m=\sigma((t,\omega)\mapsto(t,A_mX_t(\omega)))$.
The finite-observation target is the Borel field
$v_m^\star:[0,T]\times\Hm\to\Hm$ determined by
\begin{equation}
\label{eq:finite-target}
  v_m^\star(\tau,A_mX_\tau)
  =
  \Ee[A_mU_\tau\mid\cG_m].
\end{equation}
Its risk, for $f:[0,T]\times\Hm\to\Hm$, is
\begin{equation}
\label{eq:finite-risk}
  \cL_m(f)
  =
  \frac1T\int_0^T
  \Ee\norm{f(t,A_mX_t)-A_mU_t}_{\Hsp}^2\dt.
\end{equation}

\subsection{Point sensors and a regularity space}
\label{subsec:point-sensor-setting}

To use point values as observations, we assume additional spatial regularity.
Let $\Vsp$ be a separable Hilbert space continuously and densely embedded in
$\Hsp$, and suppose $X$ and $U$ satisfy Assumption~\ref{ass:interpolation}
with $\Vsp$ in place of $\Hsp$.  A sensor--reconstruction pair has the form
\[
  S_m\in\mathcal L(\Vsp,\R^{d_m}),\qquad
  R_m\in\mathcal L(\R^{d_m},\Vsp),\qquad
  A_m=R_mS_m.
\]
We require
\begin{equation}
\label{eq:regularity-reconstruction}
  \sup_m\norm{A_m}_{\mathcal L(\Vsp,\Hsp)}<\infty,
  \qquad
  \norm{A_mf-f}_{\Hsp}\longrightarrow0
  \quad(f\in\Vsp).
\end{equation}
We additionally require
$\sigma(\tau,S_mX_\tau)=\sigma(\tau,A_mX_\tau)$.  The identity
$S_mR_m=I$ is sufficient;
for oversampled least squares it is enough that $S_mR_m$ be injective on the
almost-sure range of $S_mX_t$.
On a bounded Lipschitz domain $D\subset\R^d$, take
$\Vsp=H^s(D;\R^q)$, $\Hsp=L^2(D;\R^q)$, and $s>d/2$.
Sobolev embedding makes point evaluation continuous
\cite{adams2003sobolev} and gives
$\Vsp\hookrightarrow C^{0,\alpha}$ for
$0<\alpha\le1$, $\alpha<s-d/2$, as needed in
Proposition~\ref{prop:qno-realization}. Local averages provide a bounded
alternative to point values.

\section{Continuum background}
\label{sec:continuum}

The regression identities below identify the training target and its
continuity equation. The superposition theorem
\cite[Theorem~3.4]{stepanov2017superposition}\cite{zhang2026functional}
then represents the population law by characteristic curves. Here $C_b^1(\Hsp)$ consists of bounded $C^1$
functions with bounded continuous Fr\'echet derivatives.

\begin{theorem}[Continuum flow-matching identities]
\label{thm:continuum}
Under Assumption~\ref{ass:interpolation}, the following statements hold.
\begin{enumerate}[label=(\roman*),leftmargin=2.2em]
  \item The curve $t\mapsto\mu_t$ belongs to
  $AC^2([0,T];\cP_2(\Hsp))$ and, for $0\le s\le t\le T$,
  \begin{equation}
  \label{eq:w2-ac-bound}
    \Wtwo(\mu_s,\mu_t)
    \le
    \bigl(\Ee\norm{X_t-X_s}_{\Hsp}^2\bigr)^{1/2}
    \le
    \int_s^t
    \bigl(\Ee\norm{U_r}_{\Hsp}^2\bigr)^{1/2}\dd r .
  \end{equation}
  \item The conditional field is square integrable and obeys
  \begin{equation}
  \label{eq:conditional-jensen}
    \frac1T\int_0^T
    \int_{\Hsp}\norm{v^\star(t,x)}_{\Hsp}^2\,\mu_t(\dd x)\dt
    \le
    \frac1T\int_0^T\Ee\norm{U_t}_{\Hsp}^2\dt .
  \end{equation}
  \item For every admissible $w$,
  \begin{equation}
  \label{eq:population-pythagoras}
    \cL(w)
    =
    \cL(v^\star)
    +
    \frac1T\int_0^T
    \Ee\norm{w(t,X_t)-v^\star(t,X_t)}_{\Hsp}^2\dt .
  \end{equation}
  Hence $v^\star$ is the unique minimizer in
  $L^2(\dd t\,\mu_t;\Hsp)$.
  \item The pair $(\mu_t,v^\star)$ satisfies the weak continuity equation:
  for every $\varphi\in C_b^1(\Hsp)$ and
  $\zeta\in C_c^1((0,T))$,
  \begin{equation}
  \label{eq:weak-continuity}
    \int_0^T\!\int_{\Hsp}
    \left[
      \zeta'(t)\varphi(x)
      +
      \zeta(t)\inner{D\varphi(x)}{v^\star(t,x)}_{\Hsp}
    \right]\mu_t(\dd x)\dt
    =0.
  \end{equation}
\end{enumerate}
\end{theorem}

\begin{proof}
Coupling $(X_s,X_t)$ and applying Minkowski to
$X_t-X_s=\int_s^tU_r\dd r$ proves \eqref{eq:w2-ac-bound} and $AC^2$.
Conditional Jensen gives \eqref{eq:conditional-jensen}; conditioning the
cross term in the squared loss gives \eqref{eq:population-pythagoras}.
The pathwise chain rule, Fubini, and conditioning give
\begin{equation}
\label{eq:integrated-ce}
 \Ee\varphi(X_t)-\Ee\varphi(X_s)
 =\int_s^t\!\int_{\Hsp}
   \inner{D\varphi(x)}{v^\star(r,x)}_{\Hsp}\,\mu_r(\dd x)\dd r.
\end{equation}
Bounded $D\varphi$ and $U\in L^1$ justify the interchanges; integration
by parts yields \eqref{eq:weak-continuity}.
\end{proof}

\begin{corollary}[Superposition representation]
\label{cor:superposition}
Under Assumption~\ref{ass:interpolation}, there is a Borel probability
$\eta$ on $C([0,T];\Hsp)$ with $(e_t)_\#\eta=\mu_t$ for $0\le t\le T$,
where $e_t(\gamma)=\gamma_t$. It is concentrated on absolutely continuous
curves satisfying
\begin{equation}
\label{eq:characteristic}
  \dot\gamma_t=v^\star(t,\gamma_t)
  \quad\text{for }\eta\text{-almost every }\gamma
  \text{ and almost every }t.
\end{equation}
\end{corollary}

\begin{proof}
Narrow continuity, \eqref{eq:weak-continuity}, and
$v^\star\in L^1(\dd t\,\mu_t)$ follow from Theorem~\ref{thm:continuum}.
Apply the Hilbert-space superposition theorem
\cite[Theorem~16 and Corollary~17]{zhang2026functional}
\cite[Theorem~3.4]{stepanov2017superposition}.
\end{proof}

Regression specifies an almost-everywhere equivalence class. To obtain a
unique flow map, we require a regular off-support representative.

\begin{assumption}[Regular ODE representative]
\label{ass:continuum-lipschitz}
A specified Borel representative of $v^\star$ is measurable in $t$ for each
$x$, continuous in $x$ for almost every $t$, and there are
$L,g\in L^1(0,T)$ such that
\[
  \norm{v^\star(t,x)-v^\star(t,y)}_{\Hsp}
  \le L(t)\norm{x-y}_{\Hsp},
  \qquad
  \norm{v^\star(t,0)}_{\Hsp}\le g(t)
\]
for all $x,y\in\Hsp$ and almost every $t$.
\end{assumption}

\begin{theorem}[Deterministic continuum flow]
\label{thm:deterministic-continuum-flow}
Under the interpolation and regular-ODE assumptions, the Carath\'eodory
equation
\begin{equation}
\label{eq:continuum-ode}
  \dot Y_t=v^\star(t,Y_t),\qquad Y_0=x,
\end{equation}
has a unique global solution $\Phi_t(x)$ for every $x\in\Hsp$, with
$\mu_t=(\Phi_t)_\#\mu_0$ for $0\le t\le T$ and
\begin{equation}
\label{eq:continuum-flow-lipschitz}
  \norm{\Phi_t(x)-\Phi_t(y)}_{\Hsp}
  \le e^{\int_0^tL(r)\dd r}\norm{x-y}_{\Hsp}.
\end{equation}
\end{theorem}

\begin{proof}
The Banach-space Carath\'eodory theorem gives a unique global solution, and
Gr\"onwall gives \eqref{eq:continuum-flow-lipschitz}.  Disintegrating the
measure in Corollary~\ref{cor:superposition} over $e_0$ and using uniqueness
forces each conditional path law to be
$\delta_{\Phi_\cdot(x)}$; hence
$\mu_t=(e_t)_\#\eta=(\Phi_t)_\#\mu_0$.
\end{proof}

\section{Finite observations and conditional-target consistency}
\label{sec:observations}

Each reconstructed problem has the same population interpretation.
Its conditional target converges without requiring nested observations.

\begin{proposition}[Finite reconstructed continuity equation]
\label{prop:finite-continuity}
Under Assumption~\ref{ass:interpolation}, for every $m$:
\begin{enumerate}[label=(\roman*),leftmargin=2.2em]
  \item $v_m^\star$ uniquely minimizes $\cL_m$ in
  $L^2(\dd t\,\mu_t^m;\Hm)$, and
  \begin{equation}
  \label{eq:finite-pythagoras}
    \cL_m(f)-\cL_m(v_m^\star)
    =
    \frac1T\int_0^T
    \Ee\norm{f(t,A_mX_t)-v_m^\star(t,A_mX_t)}_{\Hsp}^2\dt .
  \end{equation}
  \item $(\mu_t^m,v_m^\star)$ satisfies the weak continuity equation on
  $\Hm$.
  \item There is a superposition measure $\eta_m$ on
  $C([0,T];\Hm)$ with $(e_t)_\#\eta_m=\mu_t^m$, concentrated on
  absolutely continuous solutions of
  $\dot z_t=v_m^\star(t,z_t)$.
\end{enumerate}
\end{proposition}

\begin{proof}
Apply Theorem~\ref{thm:continuum} and Corollary~\ref{cor:superposition}
to the absolutely continuous $\Hm$-valued path $A_mX$, whose derivative
$A_mU$ is square integrable. Its conditional velocity is
\eqref{eq:finite-target}.
\end{proof}

\subsection{Qualitative and quantitative target convergence}

We use the normalized product-space norm
\[
  \norm{Z}_{\overline L^2}
  :=
  \left(\frac1T\int_0^T\Ee\norm{Z_t}_{\Hsp}^2\dt\right)^{1/2}.
\]

\begin{theorem}[Conditional-target consistency]
\label{thm:target-consistency}
Under the interpolation and reconstruction assumptions,
\begin{equation}
\label{eq:target-consistency}
  \norm{
    v_m^\star(\tau,A_mX_\tau)
    -
    v^\star(\tau,X_\tau)
  }_{\overline L^2}
  \longrightarrow0.
\end{equation}
The operators and their observation sigma-algebras need not be nested.
\end{theorem}

\begin{proof}
Write $V=v^\star(\tau,X_\tau)$ and
$M_m=\Ee[V\mid\cG_m]$ on the product probability space.  Since
$\cG_m\subseteq\cG$, the tower property and commutation of a bounded linear
operator with Bochner conditional expectation (verified by testing against
continuous linear functionals) give
\begin{equation}
\label{eq:finite-target-factorization}
  v_m^\star(\tau,A_mX_\tau)
  =
  A_m\Ee[U_\tau\mid\cG_m]
  =
  A_m\Ee[V\mid\cG_m]
  =
  A_mM_m.
\end{equation}

We claim that $M_m\to V$ in $\overline L^2$.  Let
$\xi_m=(\tau,A_mX_\tau)$ and $\xi=(\tau,X_\tau)$.  Strong convergence of
$A_m$ gives $\xi_m\to\xi$ almost surely.  Bounded continuous
$\Hsp$-valued functions on $[0,T]\times\Hsp$ are dense in
$L^2(\Law(\xi);\Hsp)$ by regularity, Urysohn's lemma, and simple-function
approximation.  Choose such a $g$ with
$\norm{V-g(\xi)}_{\overline L^2}<\varepsilon$.
Because $g(\xi_m)$ is $\cG_m$-measurable and conditional expectation is the
orthogonal projection onto $L^2(\cG_m;\Hsp)$,
\[
  \norm{V-M_m}_{\overline L^2}
  \le
  \norm{V-g(\xi_m)}_{\overline L^2}
  \le
  \varepsilon+
  \norm{g(\xi)-g(\xi_m)}_{\overline L^2}.
\]
The last term tends to zero by bounded convergence.  Taking the limsup and
then letting $\varepsilon\downarrow0$ proves the claim.

Finally, \eqref{eq:finite-target-factorization} yields
\begin{equation}
\label{eq:target-two-terms}
  \norm{A_mM_m-V}_{\overline L^2}
  \le C_A\norm{M_m-V}_{\overline L^2}
  +\norm{(A_m-I)V}_{\overline L^2}.
\end{equation}
The first term tends to zero.  The second tends to zero by pointwise strong
convergence and dominated convergence, since
$\norm{(A_m-I)V}\le(C_A+1)\norm V$ and $V\in\overline L^2$.
\end{proof}

\begin{corollary}[Quantitative target bounds]
\label{cor:quantitative-target}
\label{cor:orthogonal-target}
Assume the hypotheses of Theorem~\ref{thm:target-consistency} and suppose
$\norm{v^\star(t,x)-v^\star(t,y)}_{\Hsp}\le L(t)\norm{x-y}_{\Hsp}$
for all $x,y$ and almost every $t$, with measurable $L\ge0$.
Write $V=v^\star(\tau,X_\tau)$,
$V_m=v_m^\star(\tau,A_mX_\tau)$, and assume
\[
 D_m:=\frac1T\int_0^TL(t)^2\Ee\norm{A_mX_t-X_t}_{\Hsp}^2\dt<\infty
 \quad\text{for every }m.
\]
Then
\begin{equation}
\label{eq:quantitative-target-general}
 \norm{V_m-V}_{\overline L^2}
 \le\norm{(A_m-I)V}_{\overline L^2}+C_A\sqrt{D_m}.
\end{equation}
For orthogonal projections $A_m=P_m$, the sharper bound is
\begin{equation}
\label{eq:orthogonal-target-bound}
 \norm{V-V_m}_{\overline L^2}^2
 \le\norm{(I-P_m)V}_{\overline L^2}^2+D_m.
\end{equation}
\end{corollary}

\begin{proof}
Put $M_m=\Ee[V\mid\cG_m]$. Best approximation by the
$\cG_m$-measurable variable $v^\star(\tau,A_mX_\tau)$ gives
$\norm{M_m-V}_{\overline L^2}\le\sqrt{D_m}$; apply
\eqref{eq:target-two-terms}. For $A_m=P_m$,
$V-V_m=(I-P_m)V+P_m(V-M_m)$ is an orthogonal sum. Since
$P_mM_m=\Ee[P_mV\mid\cG_m]$, best approximation by
$P_mv^\star(\tau,P_mX_\tau)$ gives
$\norm{P_m(V-M_m)}_{\overline L^2}^2\le D_m$.
\end{proof}

\subsection{Marginal-law convergence}

\begin{proposition}[Uniform convergence of reconstructed laws]
\label{prop:law-convergence}
Under Assumption~\ref{ass:interpolation} and
\eqref{eq:strong-reconstruction},
\begin{equation}
\label{eq:law-canonical-bound}
  \Wtwo(\mu_t^m,\mu_t)
  \le
  \bigl(\Ee\norm{A_mX_t-X_t}_{\Hsp}^2\bigr)^{1/2}
\end{equation}
for every $t$. Moreover,
\begin{equation}
\label{eq:uniform-law-convergence}
 \sup_{0\le t\le T}\Wtwo(\mu_t^m,\mu_t)\to0.
\end{equation}
\end{proposition}

\begin{proof}
The canonical coupling $(A_mX_t,X_t)$ proves
\eqref{eq:law-canonical-bound}. Almost surely, the path image is compact;
strong convergence and uniform boundedness of $A_m$ imply
$\sup_t\norm{A_mX_t-X_t}\to0$ by a finite-net argument.
The square is dominated by
$(C_A+1)^2\sup_t\norm{X_t}^2\in L^1$ by \eqref{eq:path-sup-moment}.
Dominated convergence proves \eqref{eq:uniform-law-convergence}.
\end{proof}

\subsection{Point sensors and the finite-information limit}

No finite observation can reconstruct the infinite-dimensional unit ball
uniformly: for any linear $S:\Hsp\to\R^d$ and arbitrary $R$,
\[
  \sup_{\norm f_{\Hsp}\le1}\norm{R(Sf)-f}_{\Hsp}\ge1.
\]
Indeed, take a unit $h\in\ker S$ and apply the triangle inequality to
$h,-h$, which share observation zero.  For fixed point sensors on an open
$D$, the same bound holds over the $L^2$-unit ball in $C_c^\infty(D)$ by
choosing a smooth bump supported away from every sensor.  Rates of reconstruction therefore depend on a regularity class. We use
estimates of the form
\[
  \norm{A_mf-f}_{\Hsp}\le a_m\norm f_{\Vsp},
  \qquad a_m\downarrow0.
\]
For scattered sensors, the value of $a_m$ depends on fill distance, domain
geometry, smoothness, and the reconstruction scheme
\cite{wendland2005scattered,narcowich2005sobolev}.

\begin{corollary}[Point-sensor target consistency]
\label{cor:point-target}
In the regularity-space setting of
Section~\ref{subsec:point-sensor-setting}, suppose
\eqref{eq:regularity-reconstruction} holds and
$\sigma(\tau,S_mX_\tau)=\sigma(\tau,A_mX_\tau)$.  Define the continuum conditional
expectation in $\Vsp$ and the finite target in $\Hsp$.  Then
\[
  v_m^\star(\tau,A_mX_\tau)
  \longrightarrow
  v^\star(\tau,X_\tau)
  \quad\text{in }L^2(\overline\Pp;\Hsp).
\]
\end{corollary}

\begin{proof}
Let $\iota:\Vsp\hookrightarrow\Hsp$ denote the embedding. Lusin--Souslin
implies that embedding a $\Vsp$-valued variable preserves its sigma-algebra.
Put $V=\Ee_{\Vsp}[U_\tau\mid\cG]$ and
$M_m=\Ee_{\Vsp}[V\mid\cG_m]$. Bounded linear maps commute with
Bochner conditional expectation, so
$\iota V=v^\star(\tau,X_\tau)$ and
$v_m^\star(\tau,A_mX_\tau)=A_mM_m$.
The continuous-approximation argument of
Theorem~\ref{thm:target-consistency}, using $\Vsp$-valued approximants on
$[0,T]\times\Hsp$, gives $M_m\to V$ in $L^2(\overline\Pp;\Vsp)$.
With $C_{VH}=\sup_m\norm{A_m}_{\mathcal L(\Vsp,\Hsp)}$,
\[
 \norm{A_mM_m-\iota V}_{L^2(\Hsp)}
 \le C_{VH}\norm{M_m-V}_{L^2(\Vsp)}
       +\norm{A_mV-\iota V}_{L^2(\Hsp)}\to0,
\]
where the last term vanishes by
\eqref{eq:regularity-reconstruction} and dominated convergence.
\end{proof}

\section{From target consistency to flow consistency}
\label{sec:flow-consistency}

Proposition~\ref{prop:finite-continuity} represents each finite population
law by characteristic paths. To compare deterministic flow maps, we now
assume a Lipschitz bound that is uniform across observation levels.

\begin{assumption}[Mesh-uniform population ODE stability]
\label{ass:finite-lipschitz}
The continuum regular-ODE assumption holds.  Every equivalence class
$v_m^\star$ also admits a specified Carath\'eodory representative on $\Hm$
with
\[
  \norm{v_m^\star(t,x)-v_m^\star(t,y)}_{\Hsp}
  \le L_m(t)\norm{x-y}_{\Hsp},
  \qquad
  \sup_m\int_0^TL_m(t)\dt\le\Lambda<\infty,
\]
together with an integrable linear-growth bound sufficient for global
existence.
\end{assumption}

\begin{remark}[Stability is not inherited]
\label{rem:lipschitz-not-inherited}
The next construction gives a linear continuum velocity whose finite
conditional targets have increasingly steep derivatives. It explains why
Assumption~\ref{ass:finite-lipschitz} is separate from the continuum
Lipschitz condition in Corollary~\ref{cor:quantitative-target}.
\end{remark}

\begin{theorem}[Failure of the mesh-uniform Gr\"onwall condition]
\label{thm:lipschitz-counterexample}
There exist an interpolation on $\Hsp=\ell^2$, a target $v^\star(x)=Bx$
with $\norm B=1$, and the standard nested orthogonal projections $P_m$
such that every finite target has a globally Lipschitz representative, but
every choice of such representatives $\widetilde v_m$ satisfies
\begin{equation}
\label{eq:counterexample-lipschitz-divergence}
 \sup_m\int_0^T\Lip(\widetilde v_m(t,\cdot))\dt=\infty
 \qquad(T>0).
\end{equation}
Nevertheless, the canonical finite representatives constructed below obey
\begin{equation}
\label{eq:counterexample-exact-flow}
 \Phi_t^m(P_mX_0)=P_m\Phi_t(X_0)
 \quad\text{almost surely for all }m,\ t\in[0,T].
\end{equation}
Thus the failure of Assumption~\ref{ass:finite-lipschitz} need not prevent
flow convergence.

\end{theorem}

\begin{proof}
Let $(e_j)$ be the canonical basis,
$Bx=\sum_{k\ge1}x_{2k}e_{2k-1}$, $a_k=2^{-k}$, $M_k=2^k$, and
$f_k(a)=a_k\tanh(M_ka/a_k)$. For independent
$Z_k\sim\operatorname{Unif}[-1,1]$, set
\[
 X_0=\sum_{k\ge1}\{a_kZ_ke_{2k-1}+f_k(a_kZ_k)e_{2k}\},
 \qquad X_t=(I+tB)X_0.
\]
The series is bounded in $\ell^2$ since $\abs{f_k}\le a_k$.
As $B^2=0$ and $U_t=BX_0=BX_t$, the continuum target is $Bx$.
For $m=2k-1$, put $A_k=a_kZ_k$ and $F_{k,t}(a)=a+tf_k(a)$.
This map is strictly increasing, so conditioning on its value gives
\begin{equation}
\label{eq:counterexample-finite-target}
 v_{2k-1}^\star(t,z)=\sum_{j<k}z_{2j}e_{2j-1}
       +g_{k,t}(z_{2k-1})e_{2k-1},
 \qquad g_{k,t}=f_k\circ F_{k,t}^{-1}
\end{equation}
on the projected support. Composing $g_{k,t}$ with projection onto its
compact support interval gives a globally Lipschitz extension, jointly
continuous in $(t,z)$ because the interval endpoints and inverse vary
continuously with $t$. Even-level targets are linear restrictions.

For any globally Lipschitz representative, Fubini and the positive density
of $F_{k,t}(A_k)$ force its last component to equal $g_{k,t}$ almost
everywhere after fixing almost every earlier pair; continuity extends this
equality to the support interval. Consequently, for almost every $t$,
\[
 \Lip(\widetilde v_{2k-1}(t,\cdot))
 \ge g_{k,t}'(0)=\frac{M_k}{1+tM_k},
 \qquad
 \int_0^T\Lip(\widetilde v_{2k-1}(t,\cdot))\dt
 \ge\log(1+TM_k)\to\infty.
\]
Finally, $\Phi_t(x)=(I+tB)x$. The canonical finite ODE is solved by
$z_{2k-1}(t)=F_{k,t}(A_k)$, since
$\dot z_{2k-1}=f_k(A_k)=g_{k,t}(z_{2k-1})$; complete pairs follow the
linear flow. This proves \eqref{eq:counterexample-exact-flow}.
\end{proof}

\begin{remark}[Scope of the counterexample]
\label{rem:counterexample-scope}
The steep derivative occurs when observation splits a pair of coordinates
coupled through $X_1=(I+B)X_0$. An independent Gaussian endpoint would
change this conditional target.
\end{remark}

\begin{theorem}[Population-law flow convergence]
\label{thm:population-flow-convergence}
Suppose the interpolation, continuum regularity, and mesh-uniform population
stability assumptions hold.  Let
\[
  Y_t=\Phi_t(X_0),\qquad
  \dot Y_t^m=v_m^\star(t,Y_t^m),\quad Y_0^m=A_mX_0.
\]
Define $r_{m,0}=\norm{(A_m-I)X_0}_{L^2(\Pp;\Hsp)}$ and
\begin{align*}
 r_{m,\mathrm{path}}&=\int_0^TL_m(t)\norm{(A_m-I)Y_t}_{L^2(\Pp;\Hsp)}\dt,\\
 r_{m,\mathrm{vel}}&=\int_0^T
 \norm{v_m^\star(t,A_mY_t)-v^\star(t,Y_t)}_{L^2(\Pp;\Hsp)}\dt.
\end{align*}
Then
\begin{equation}
  \left(
    \Ee\sup_{0\le t\le T}\norm{Y_t^m-Y_t}_{\Hsp}^2
  \right)^{1/2}
  \le
  e^\Lambda
  \bigl(r_{m,0}+r_{m,\mathrm{path}}+r_{m,\mathrm{vel}}\bigr).
  \label{eq:strong-flow-bound}
\end{equation}
Consequently,
\[
  Y^m\longrightarrow Y
  \quad\text{in }L^2(\Omega;C([0,T];\Hsp)).
\]
Moreover, $\Law(Y_t^m)=\mu_t^m$ and $\Law(Y_t)=\mu_t$ for every $t$.
\end{theorem}

\begin{proof}
Add and subtract $v_m^\star(t,A_mY_t)$ in the integral equations and use
$\norm{Y_t^m-A_mY_t}\le\norm{Y_t^m-Y_t}+\norm{(I-A_m)Y_t}$.
Gr\"onwall followed by Minkowski proves \eqref{eq:strong-flow-bound}.
By Theorem~\ref{thm:deterministic-continuum-flow},
$\Law(Y_t)=\mu_t$, so $(\tau,Y_\tau)$ and $(\tau,X_\tau)$ have the
same law. Theorem~\ref{thm:target-consistency} and Cauchy--Schwarz imply
$r_{m,\mathrm{vel}}\to0$; strong and dominated convergence give
$r_{m,0}\to0$.

Each continuous $Y$-path has compact image, on which $A_m\to I$ is
uniform. The ODE growth bound makes $\sup_t\norm{Y_t}$ square integrable,
so dominated convergence gives
\[
 \norm{\sup_{t\le T}\norm{(A_m-I)Y_t}}_{L^2}\to0.
\]
Since $r_{m,\mathrm{path}}$ is at most $\Lambda$ times this norm, it also
vanishes. Finally, disintegrate the superposition measure of
Proposition~\ref{prop:finite-continuity} over its initial state.
Uniqueness under Assumption~\ref{ass:finite-lipschitz} identifies it with
the finite flow, proving $\Law(Y_t^m)=\mu_t^m$.
\end{proof}

If the initial and uniform reconstructed-path errors are at most
$\epsilon_m$, and the normalized squared target error is at most
$\epsilon_m^2$, then \eqref{eq:strong-flow-bound} and Cauchy--Schwarz give
\[
  \left(\Ee\sup_{t\le T}\norm{Y_t^m-Y_t}^2\right)^{1/2}
  \le e^\Lambda(1+\Lambda+T)\epsilon_m.
\]

Section~\ref{sec:end-to-end} compares generated distributions using the
learned field's stability. That argument uses the population
superposition paths from Proposition~\ref{prop:finite-continuity}.

\section{Operator approximation and mesh-uniform learning}
\label{sec:learning}

At level $m$, training uses i.i.d.\ copies of
$Z_m=(\tau,A_mX_\tau)$ and $W_m=A_mU_\tau$, with
$\tau\sim\operatorname{Unif}[0,T]$.  For a field class $\cF_{m,p}$, define
\begin{equation}
\label{eq:approximation-error}
  a_{m,p}
  =
  \inf_{f\in\cF_{m,p}}
  \{\cL_m(f)-\cL_m(v_m^\star)\}
  =
  \inf_{f\in\cF_{m,p}}
  \Ee\norm{f(Z_m)-v_m^\star(Z_m)}^2.
\end{equation}

We call $\widehat v_m\in\cF_{m,p}$ an
$\varepsilon_{\rm opt}$-approximate empirical risk minimizer if
\begin{equation}
\label{eq:approximate-erm}
  \widehat\cL_{m,n}(\widehat v_m)
  \le\inf_{f\in\cF_{m,p}}\widehat\cL_{m,n}(f)
  +\varepsilon_{\rm opt}.
\end{equation}

\begin{assumption}[Uniform parametrized class]
\label{ass:parametric-class}
For every $m$, let
$\cF_{m,p}=\{f_{\theta,m}:\theta\in\Theta\}$ with nonempty
$\Theta\subseteq\{\theta\in\R^p:\norm\theta_2\le R\}$.  There are constants
$B,G\in[0,\infty)$, independent of $m$, such that, on a single event of
probability one, the following bounds hold simultaneously for all
$\theta,\theta'\in\Theta$:
\[
  \norm{f_{\theta,m}(Z_m)}_{\Hsp}\le B,\quad
  \norm{W_m}_{\Hsp}\le B,\quad
  \norm{f_{\theta,m}(Z_m)-f_{\theta',m}(Z_m)}_{\Hsp}
  \le G\norm{\theta-\theta'}_2.
\]
\end{assumption}

Here $B$ bounds outputs and targets, and $G$ controls parameter
sensitivity. Gaussian data require clipping
(Proposition~\ref{prop:gaussian-clipping}) or moment-based estimates.

\subsection{Verification for a quadrature neural operator}

We now parameterize the flow-matching velocity field by a neural operator
and verify the uniform bounds required above. At each time, the operator
maps the observed function values to predicted velocities at the same
sensors. Its parameters are shared across sensor sets.

Let $x_{m,1},\ldots,x_{m,d_m}$ have positive weights summing to one, and set
$\norm y_m^2=\sum_jw_{m,j}\norm{y_j}_2^2$.  Consider
\begin{align}
  h_i^0&=h^0(t,x_{m,i},x_i),\nonumber\\
  h_i^{\ell+1}
  &=
  \sigma_\ell\!\left(
    B_{\ell,\theta}h_i^\ell+
    \sum_jw_{m,j}
    \kappa_{\ell,\theta}(x_{m,i},x_{m,j},t)h_j^\ell+
    b_{\ell,\theta}(x_{m,i},t)
  \right),\label{eq:concrete-qno}\\
  f_{\theta,m}(t,x)_i&=C_\theta h_i^L.\nonumber
\end{align}

Let $\mathcal Y_m=(\R^q)^{d_m}$ and let
$\iota_m:\mathcal Y_m\to\Hm$ be a linear isomorphism.
To compare nodal and risk norms, assume
\begin{equation}
\label{eq:qno-mass-stability}
  0<c_{\rm mass}\le C_{\rm mass}<\infty,
  \qquad
  c_{\rm mass}\norm y_m
  \le\norm{\iota_my}_{\Hsp}
  \le C_{\rm mass}\norm y_m
  \quad(y\in\mathcal Y_m),
\end{equation}
independently of $m$.  The components in \eqref{eq:concrete-qno} define
$\mathbf f_{\theta,m}$, with
$f_{\theta,m}(t,\iota_mx)=\iota_m\mathbf f_{\theta,m}(t,x)$ and
$\mathbf W_m=\iota_m^{-1}W_m$.  This is exact for cellwise realizations with cell-mass weights
($c_{\rm mass}=C_{\rm mass}=1$) and is the usual stable finite-element
mass-matrix equivalence.

\begin{theorem}[Uniform operator constants]
\label{thm:qno-uniform-constants}
Assume $\norm{h^0}_m\le H_0$ almost surely for a constant independent of
$m$, and let $h^0$ be independent of $\theta$.  For
$\ell<L$, suppose
\begin{align*}
 \norm{B_{\ell,\theta}}_{\rm op}&\le\overline B_\ell,
 &\norm{B_{\ell,\theta}-B_{\ell,\theta'}}_{\rm op}
 &\le\beta_\ell\norm{\theta-\theta'}_2,\\
 \sup_{x,y,t}\norm{\kappa_{\ell,\theta}}_{\rm op}
 &\le\overline K_\ell,
 &\sup_{x,y,t}\norm{\kappa_{\ell,\theta}-\kappa_{\ell,\theta'}}_{\rm op}
 &\le\gamma_\ell\norm{\theta-\theta'}_2,\\
 \sup_{x,t}\norm{b_{\ell,\theta}-b_{\ell,\theta'}}_2
 &\le\eta_\ell\norm{\theta-\theta'}_2.
\end{align*}
Let $\sigma_\ell$ be globally $s_\ell$-Lipschitz.  Starting from $H_0$,
suppose that, for every $\ell<L$, one of the following two envelope
conditions holds:
(i) $\sup_z\norm{\sigma_\ell(z)}_2\le H_{\ell+1}$; or
(ii) $\sup_{\theta,x,t}\norm{b_{\ell,\theta}(x,t)}_2\le A_\ell<\infty$
and, with $c_{\sigma,\ell}=\norm{\sigma_\ell(0)}_2$,
\begin{equation}
\label{eq:qno-magnitude-recurrence}
 H_{\ell+1}=c_{\sigma,\ell}
     +s_\ell\{(\overline B_\ell+\overline K_\ell)H_\ell+A_\ell\}.
\end{equation}
Assume also
$\norm{C_\theta}_{\rm op}\le\overline C$ and
$\norm{C_\theta-C_{\theta'}}_{\rm op}
\le\chi\norm{\theta-\theta'}_2$.
With $D_0=0$ and
\begin{equation}
\label{eq:qno-parameter-recurrence}
  D_{\ell+1}
  =
  s_\ell\{
    (\overline B_\ell+\overline K_\ell)D_\ell
    +(\beta_\ell+\gamma_\ell)H_\ell+\eta_\ell
  \},
\end{equation}
one has, uniformly in $m$,
\begin{align}
\label{eq:qno-output-bound}
 \norm{f_{\theta,m}(Z_m)}_{\Hsp}&\le C_{\rm mass}\overline C H_L,\\
\label{eq:qno-parameter-bound}
 \norm{f_{\theta,m}(Z_m)-f_{\theta',m}(Z_m)}_{\Hsp}
 &\le C_{\rm mass}(\overline C D_L+\chi H_L)\norm{\theta-\theta'}_2.
\end{align}
Hence Assumption~\ref{ass:parametric-class} holds with
\[
  B=C_{\rm mass}\max\{M_U,\overline C H_L\},
  \qquad
  G=C_{\rm mass}(\overline C D_L+\chi H_L),
\]
whenever $\norm{\mathbf W_m}_m\le M_U$ almost surely.

If, in nodal coordinates,
$\norm{h^0(t,x)-h^0(t,x')}_m\le J_0\norm{x-x'}_m$ and
$J_{\ell+1}=s_\ell(\overline B_\ell+\overline K_\ell)J_\ell$, then
\begin{equation}
\label{eq:qno-input-lipschitz}
  \norm{f_{\theta,m}(t,z)-f_{\theta,m}(t,z')}_{\Hsp}
  \le
  \frac{C_{\rm mass}}{c_{\rm mass}}\,
  \overline C J_L\norm{z-z'}_{\Hsp},
  \qquad z,z'\in\Hm.
\end{equation}
\end{theorem}

\begin{proof}
Normalized weights and Jensen give
$\norm{\sum_jw_{m,j}\kappa_{ij}r_j}_2\le\overline K_\ell\norm r_m$.
Thus bounded activations give envelope (i), while under (ii) the
preactivation satisfies
$\norm{q^\ell}_m\le(\overline B_\ell+\overline K_\ell)
\norm{h^\ell}_m+A_\ell$ and
$\norm{h^{\ell+1}}_m\le c_{\sigma,\ell}+s_\ell\norm{q^\ell}_m$.
Induction proves \eqref{eq:qno-magnitude-recurrence}.

The preactivation difference for $\theta,\theta'$ is bounded by
$\{(\overline B_\ell+\overline K_\ell)D_\ell+
(\beta_\ell+\gamma_\ell)H_\ell+\eta_\ell\}
\norm{\theta-\theta'}_2$. Applying the activation and output layer
proves the nodal $D_\ell$, output, and parameter bounds. With fixed
parameters the same argument gives $J_\ell$. The upper mass bound
transfers output and parameter estimates to $\Hsp$; both bounds give
\[
 \norm{f_{\theta,m}(t,z)-f_{\theta,m}(t,z')}_{\Hsp}
 \le C_{\rm mass}\overline C J_L
       \norm{\iota_m^{-1}(z-z')}_m
 \le\frac{C_{\rm mass}}{c_{\rm mass}}\overline C J_L\norm{z-z'}_{\Hsp}.
\]
\end{proof}

The theorem applies to componentwise ReLU, exact GeLU $z\Phi(z)$,
and SiLU $z/(1+e^{-z})$: ReLU is $1$-Lipschitz, and the other two have
bounded derivatives.

\paragraph{Affine dictionaries.}
Let $T_{\ell,\theta}=T_\ell^0+\sum_{r\le p}\theta_rT_\ell^r$ be affine.
Using the relevant uniform operator or Euclidean norm, set
$\Gamma_{T,\ell}=(\sum_{r\le p}\norm{T_\ell^r}_*^2)^{1/2}$.
Cauchy--Schwarz gives
\[
  \sup_\theta\norm{T_{\ell,\theta}}_*
  \le\norm{T_\ell^0}_*+R\Gamma_{T,\ell},
  \qquad
  \norm{T_{\ell,\theta}-T_{\ell,\theta'}}_*
  \le\Gamma_{T,\ell}\norm{\theta-\theta'}_2.
\]
These estimates give sensor-independent constants for $B,\kappa,b,C$.
Their dependence on model size is different. At fixed depth, with
$M_U,\overline C,\chi=O(1)$ and atoms of order-one norm, bounded
activations give $B_p=O(1)$ and $G_p=O(p^{L/2})$.
The magnitude recurrence gives $B_p,G_p=O(p^{L/2})$.
If $\chi=O(\sqrt p)$ instead, then $G_p=O(p^{(L+1)/2})$.
Normalized dictionaries give $B_p,G_p=O(1)$.
In each case, $\sum_jw_{m,j}=1$ removes dependence on the number of sensors.

\begin{theorem}[Mesh-uniform fast oracle bound]
\label{thm:fast-oracle}
Under Assumption~\ref{ass:parametric-class}, suppose
$\widehat v_m\in\cF_{m,p}$ satisfies the approximate ERM condition
\eqref{eq:approximate-erm}.  For $n\ge1$, $0<\delta<1$, and
$q_{p,n,\delta}=p\log(1+2n)+\log(2/\delta)$, with probability at least
$1-\delta$,
\begin{equation}
\label{eq:fast-oracle}
  \cL_m(\widehat v_m)-\cL_m(v_m^\star)
  \le
  2a_{m,p}
  +\frac32\varepsilon_{\rm opt}
  +\frac{10BGR}{n}
  +\frac{80B^2q_{p,n,\delta}}{n}.
\end{equation}
No convexity, well-specification, or attainment of the infimum defining
$a_{m,p}$ is required.  In the realizable case
$v_m^\star\in\cF_{m,p}$, so that $a_{m,p}=0$, the last constant improves:
\begin{equation}
\label{eq:fast-oracle-realizable}
  \cL_m(\widehat v_m)-\cL_m(v_m^\star)
  \le
  \frac32\varepsilon_{\rm opt}
  +\frac{10BGR}{n}
  +\frac{40B^2q_{p,n,\delta}}{n}.
\end{equation}
\end{theorem}

\begin{proof}
Let $P,P_n$ denote population and empirical averages, $v=v_m^\star$, and
\[
 g_\theta=\norm{f_{\theta,m}(Z_m)-W_m}^2
           -\norm{v(Z_m)-W_m}^2,
 \qquad r_\theta=Pg_\theta=\Ee\norm{f_{\theta,m}(Z_m)-v(Z_m)}^2.
\]
The last identity is finite Pythagoras. Conditional Jensen gives
$\norm{v(Z_m)}\le B$. Factoring the difference of squares as
$\inner{f_{\theta,m}-v}{f_{\theta,m}+v-2W_m}$ yields
\[
 Pg_\theta^2\le16B^2r_\theta,
 \qquad \abs{g_\theta-Pg_\theta}\le8B^2.
\]
Bernstein's inequality \cite[Sections~2.7--2.8]{boucheron2013concentration}
bounds either one-sided deviation by
$\sqrt{32B^2r_\theta q/n}+8B^2q/(3n)$ with failure probability $e^{-q}$.
Using $\sqrt{32B^2r_\theta q/n}\le r_\theta/3+24B^2q/n$ gives the bound
$r_\theta/3+80B^2q/(3n)$.

An $R/n$-net of $\Theta$ has at most $(1+2n)^p$ points, and the loss is
$4BG$-Lipschitz in $\theta$. Union-bounding both signs and transferring
from the net gives, simultaneously for all $\theta$,
\begin{equation}
\label{eq:fast-net-transfer}
 r_\theta\le\frac32P_ng_\theta
       +\frac{40B^2q_{p,n,\delta}}n+\frac{10BGR}n.
\end{equation}
Choose $\widehat\theta$ realizing $\widehat v_m$ and a deterministic
$\eta$-minimizer $r_{\theta_\eta}\le a_{m,p}+\eta$. The opposite
Bernstein tail and approximate ERM give on the same event
\[
 P_ng_{\widehat\theta}\le P_ng_{\theta_\eta}+\varepsilon_{\rm opt}
 \le\frac43r_{\theta_\eta}
       +\frac{80B^2q_{p,n,\delta}}{3n}+\varepsilon_{\rm opt}.
\]
Substitution into \eqref{eq:fast-net-transfer} and a deterministic sequence
$\eta\downarrow0$ prove \eqref{eq:fast-oracle}. If the class contains
$v$, then $g_{\theta_\star}=0$ pointwise; comparison with
$\theta_\star$ removes the comparator deviation and gives
\eqref{eq:fast-oracle-realizable}.
\end{proof}

\begin{remark}[Statistical rate]
\label{rem:fast-rate}
For fixed $p,B,G,R$, the estimation contribution in
\eqref{eq:fast-oracle} is $\widetilde O(n^{-1})$.  In the realizable case
with $\varepsilon_{\rm opt}=\widetilde O(n^{-1})$, this is also the total
excess risk, giving a $\widetilde O(n^{-1/2})$ statistical contribution in
Theorem~\ref{thm:learned-law-stability}.
\end{remark}

\subsection{Continuum realization and approximation}

\begin{proposition}[Objective consistency]
\label{prop:objective-consistency}
\label{prop:approximation-decomposition}
Let $\Theta$ be nonempty, let $f_\theta$ be continuum fields, and let
$f_{\theta,m}$ be their realizations with
$\cF_{m,p}=\{f_{\theta,m}:\theta\in\Theta\}$.
Suppose
\begin{align}
\label{eq:operator-realization-error}
 d_m&:=
 \sup_\theta
 \norm{f_{\theta,m}(\tau,A_mX_\tau)-f_\theta(\tau,X_\tau)}
 _{\overline L^2}\to0,\\
\label{eq:target-projection-error}
 r_m&:=\norm{A_mU_\tau-U_\tau}_{\overline L^2}\to0,
\end{align}
and
$M=\sup_\theta
\norm{f_\theta(\tau,X_\tau)-U_\tau}_{\overline L^2}<\infty$.
Then
\begin{equation}
\label{eq:uniform-objective-consistency}
  \sup_\theta
  \abs{\cL_m(f_{\theta,m})-\cL(f_\theta)}
  \le(d_m+r_m)(2M+d_m+r_m)\to0.
\end{equation}
Moreover, with
\[
  e_m=
  \norm{v_m^\star(\tau,A_mX_\tau)-v^\star(\tau,X_\tau)}
  _{\overline L^2},
  \qquad
  a_p^\infty=
  \inf_\theta
  \norm{f_\theta(\tau,X_\tau)-v^\star(\tau,X_\tau)}
  _{\overline L^2}^2,
\]
\begin{equation}
\label{eq:approximation-decomposition-bound}
  a_{m,p}\le(d_m+\sqrt{a_p^\infty}+e_m)^2.
\end{equation}
\end{proposition}

\begin{proof}
Set $E_m=f_{\theta,m}(\tau,A_mX_\tau)-A_mU_\tau$ and
$E=f_\theta(\tau,X_\tau)-U_\tau$.  Then
$\norm{E_m-E}_{\overline L^2}\le d_m+r_m$ and
$\norm{E_m}_{\overline L^2}\le M+d_m+r_m$.
Cauchy--Schwarz applied to
$\abs{\norm a^2-\norm b^2}\le\norm{a-b}(\norm a+\norm b)$
proves \eqref{eq:uniform-objective-consistency}.  Also,
\[
  \norm{f_{\theta,m}-v_m^\star}_{\overline L^2}
  \le d_m+\norm{f_\theta-v^\star}_{\overline L^2}+e_m.
\]
Taking the infimum and squaring proves the second claim.
\end{proof}

\begin{proposition}[Realization rate for the quadrature operator]
\label{prop:qno-realization}
Let $D$ be a compact metric space with metric $\mathfrak d$.  Let $\nu$ be
a Borel probability measure on $D$, and work in the regularity-space setting of
Section~\ref{subsec:point-sensor-setting} with
\[
  \Hsp=L^2(D,\nu;\R^q),
  \qquad
  \Vsp\hookrightarrow C^{0,\alpha}(D;\R^q),
  \qquad 0<\alpha\le1.
\]
Write $[g]_\alpha=\sup_{\xi\ne\eta}
\norm{g(\xi)-g(\eta)}_2/\mathfrak d(\xi,\eta)^\alpha$ and set
\[
  S_mg=(g(x_{m,1}),\ldots,g(x_{m,d_m})),
  \qquad
  A_m=\iota_mS_m,
  \qquad
  \nu_m=\sum_iw_{m,i}\delta_{x_{m,i}}.
\]
Assume $\iota_m(\mathcal Y_m)=\Hm\subset\Vsp$ and
\eqref{eq:qno-mass-stability}. Let
$q_m=W_{1,\mathfrak d^\alpha}(\nu_m,\nu)\to0$, where the metric is
$(\xi,\eta)\mapsto\mathfrak d(\xi,\eta)^\alpha$.
Suppose also that, for some $\eta_m\to0$,
\begin{equation}
\label{eq:qno-reconstruction-rate}
  \norm{\iota_mS_mg-g}_{\Hsp}
  \le
  \eta_m\bigl(\norm{g}_\infty+[g]_\alpha\bigr)
\end{equation}
for every $\R^q$-valued $\alpha$-H\"older function $g$.

Define the continuum realization of \eqref{eq:concrete-qno} by
\begin{align*}
  h_\theta^0(t,\xi;x)
  &=
  h^0(t,\xi,x(\xi)),\\
  h_\theta^{\ell+1}(t,\xi;x)
  &=
  \sigma_\ell\left(
    B_{\ell,\theta}h_\theta^\ell(t,\xi;x)
    +
    \int_D
    \kappa_{\ell,\theta}(\xi,\eta,t)
    h_\theta^\ell(t,\eta;x)\,\nu(\dd\eta)
    +
    b_{\ell,\theta}(\xi,t)
  \right),\\
  f_\theta(t,x)(\xi)
  &=
  C_\theta h_\theta^L(t,\xi;x).
\end{align*}
In addition to the operator and activation bounds of
Theorem~\ref{thm:qno-uniform-constants}, suppose, uniformly in $\theta$ and
$t$,
\begin{align*}
  \norm{
    \kappa_{\ell,\theta}(\xi,\eta,t)
    -
    \kappa_{\ell,\theta}(\xi',\eta',t)
  }_{\rm op}
  &\le
  K_{\ell,1}\mathfrak d(\xi,\xi')^\alpha
  +
  K_{\ell,2}\mathfrak d(\eta,\eta')^\alpha,\\
  [b_{\ell,\theta}(\cdot,t)]_\alpha
  &\le L_{b,\ell},\\
  \norm{b_{\ell,\theta}(\cdot,t)}_\infty
  &\le A_\ell.
\end{align*}
For $(t,\omega)$ distributed according to $\overline\Pp$, set
\[
  \mathsf H_0
  =
  \norm{h^0(t,\cdot,X_t(\cdot))}_\infty,
  \qquad
  \mathsf P_0
  =
  [h^0(t,\cdot,X_t(\cdot))]_\alpha,
\]
and assume $\mathsf H_0,\mathsf P_0\in
L^2(\overline\Pp)$.  Recursively define
\begin{align}
\label{eq:qno-spatial-envelope}
  \mathsf H_{\ell+1}
  &=
  c_{\sigma,\ell}
  +
  s_\ell\{
    (\overline B_\ell+\overline K_\ell)\mathsf H_\ell
    +A_\ell
  \},\nonumber\\
  \mathsf P_{\ell+1}
  &=
  s_\ell\{
    \overline B_\ell\mathsf P_\ell
    +K_{\ell,1}\mathsf H_\ell
    +L_{b,\ell}
  \},
\end{align}
and
\begin{align}
\label{eq:qno-realization-recurrence}
  \mathsf E_{m,0}
  &=0,\nonumber\\
  \mathsf E_{m,\ell+1}
  &=
  s_\ell\left\{
    (\overline B_\ell+\overline K_\ell)\mathsf E_{m,\ell}
    +
    q_m\bigl(
      \overline K_\ell\mathsf P_\ell
      +K_{\ell,2}\mathsf H_\ell
    \bigr)
  \right\}.
\end{align}
Then
\begin{equation}
\label{eq:qno-realization-bound}
  d_m
  \le
  \overline C\left[
    C_{\rm mass}
    \norm{\mathsf E_{m,L}}_{L^2(\overline\Pp)}
    +
    \eta_m
    \left(
      \norm{\mathsf H_L}_{L^2(\overline\Pp)}
      +
      \norm{\mathsf P_L}_{L^2(\overline\Pp)}
    \right)
  \right].
\end{equation}
In particular, $d_m\to0$ at fixed depth whenever $q_m,\eta_m\to0$.
Moreover,
\[
  \sup_\theta
  \norm{f_\theta(\tau,X_\tau)}_{\overline L^2}
  \le
  \overline C
  \norm{\mathsf H_L}_{L^2(\overline\Pp)},
\]
so the finiteness condition $M<\infty$ in
Proposition~\ref{prop:objective-consistency} follows from
$U_\tau\in L^2(\overline\Pp;\Hsp)$.
\end{proposition}

\begin{proof}
Induction gives \eqref{eq:qno-spatial-envelope}: normalization bounds the
integral by $\overline K_\ell\mathsf H_\ell$, while spatial variation of
the local term, first kernel argument, and bias contributes
$\overline B_\ell\mathsf P_\ell+K_{\ell,1}\mathsf H_\ell+L_{b,\ell}$.
The $s_\ell$-Lipschitz activation preserves these bounds.

For fixed $\xi$, the integrand
$\kappa_{\ell,\theta}(\xi,\eta,t)h_\theta^\ell(t,\eta;X_t)$ has
H\"older seminorm at most
$\overline K_\ell\mathsf P_\ell+K_{\ell,2}\mathsf H_\ell$.
A coupling of $\nu_m,\nu$ therefore bounds its quadrature error by
$q_m(\overline K_\ell\mathsf P_\ell+K_{\ell,2}\mathsf H_\ell)$.
The normalized-kernel Jensen estimate and activation bound give
\eqref{eq:qno-realization-recurrence}. Its initial error is zero since
$\iota_m^{-1}A_mX_t=S_mX_t$.

Add and subtract $\iota_mS_mf_\theta(t,X_t)$. Mass stability and
\eqref{eq:qno-reconstruction-rate} give
\[
 \norm{f_{\theta,m}(t,A_mX_t)-f_\theta(t,X_t)}_{\Hsp}
 \le\overline C\{C_{\rm mass}\mathsf E_{m,L}
                  +\eta_m(\mathsf H_L+\mathsf P_L)\}.
\]
Taking $L^2$ norms and the supremum over $\theta$ proves
\eqref{eq:qno-realization-bound}; the output envelope follows from
$\norm{C_\theta}_{\rm op}\le\overline C$.
\end{proof}

Cell-mass quadrature on cells of diameter $h_m$, with a stable
$C^{0,\alpha}$ reconstruction estimate, gives
$q_m,\eta_m,d_m=O(h_m^\alpha)$ at fixed depth. Input regularity thus
controls both cubature and reconstruction at every layer.

For a Gaussian source on $\Hsp$, the covariance must be trace class.
Spatial white noise has covariance $I$ and infinite trace, so it is not an
$L^2$-valued source \cite{daprato2014stochastic}.
\section{Error in the generated distribution}
\label{sec:end-to-end}

We compare a learned trajectory with a characteristic path drawn from the
finite population superposition measure. This coupling avoids choosing a
unique population flow. Throughout this section, we condition on the
training sigma-algebra $\cD_{m,n}$ and use fresh randomness for sampling.
We first bound the error for each trained field, then apply the statistical
risk bound.

\begin{assumption}[Stable learned field]
\label{ass:learned-lipschitz}
The realized $\widehat v_m$ is a Carath\'eodory field on $\Hm$ with an
integrable linear-growth bound, and
\[
  \norm{\widehat v_m(t,x)-\widehat v_m(t,y)}_{\Hsp}
  \le \widehat L_m(t)\norm{x-y}_{\Hsp},
  \qquad
  \int_0^T\widehat L_m(t)\dt\le\widehat\Lambda,
\]
where $\widehat\Lambda$ is independent of $m$.
\end{assumption}

Write $\widehat\Phi_{m,t}$ for the unique learned flow and
$\widehat\mu_{m,t}=(\widehat\Phi_{m,t})_\#\mu_0^m$.

\begin{theorem}[Excess risk controls the generated law]
\label{thm:learned-law-stability}
Under Assumption~\ref{ass:learned-lipschitz},
\begin{equation}
\label{eq:learned-law-finite}
  \Wtwo(\widehat\mu_{m,T},\mu_T^m)
  \le
  T e^{\widehat\Lambda}
  \left[
    \cL_m(\widehat v_m)-\cL_m(v_m^\star)
  \right]^{1/2}.
\end{equation}
Consequently,
\begin{equation}
\label{eq:learned-law-continuum}
  \Wtwo(\widehat\mu_{m,T},\mu_T)
  \le
  \norm{A_mX_T-X_T}_{L^2(\Pp;\Hsp)}
  +
  T e^{\widehat\Lambda}
  \left[
    \cL_m(\widehat v_m)-\cL_m(v_m^\star)
  \right]^{1/2}.
\end{equation}
\end{theorem}

\begin{proof}
Draw $Z\sim\eta_m$ from Proposition~\ref{prop:finite-continuity} and solve
the learned ODE from $\widehat Z_0=Z_0$. Then $Z_t\sim\mu_t^m$ and
$\widehat Z_T\sim\widehat\mu_{m,T}$. Gr\"onwall, Minkowski, and
Cauchy--Schwarz yield
\[
 \norm{\widehat Z_T-Z_T}_{L^2(\eta_m)}
 \le e^{\widehat\Lambda}\int_0^T
       \norm{\widehat v_m(t,\cdot)-v_m^\star(t,\cdot)}_{L^2(\mu_t^m)}\dt
 \le Te^{\widehat\Lambda}
       \{\cL_m(\widehat v_m)-\cL_m(v_m^\star)\}^{1/2},
\]
where the last step uses \eqref{eq:finite-pythagoras}.
This coupling proves \eqref{eq:learned-law-finite}; the triangle inequality
and \eqref{eq:law-canonical-bound} give \eqref{eq:learned-law-continuum}.
\end{proof}

\begin{remark}
The learned field's Lipschitz bound controls how trajectory errors grow.
The training loss measures velocity error along the population law;
learned trajectories can leave that region. The Lipschitz assumption
controls how the two trajectories separate \cite{benton2024error}.
\end{remark}

\subsection{Numerical and endpoint errors}

Let $\widetilde\mu_{m,T}$ be the law produced by a numerical ODE solver.
We express its accuracy by a coupling with the exact learned flow and
retain any dependence on the observation level $m$.

\begin{assumption}[Coupled solver error]
\label{ass:solver}
There is a coupling of the numerical endpoint $\widetilde Z_T$ and the exact
learned-flow endpoint $\widehat Z_T$ such that
\begin{equation}
\label{eq:solver-error}
  \left(
    \Ee[
      \norm{\widetilde Z_T-\widehat Z_T}_{\Hsp}^2
      \mid\cD_{m,n}
    ]
  \right)^{1/2}
  \le\varepsilon_{{\rm sol},m}(h),
\end{equation}
almost surely in the training data, where $h$ denotes the numerical step or
tolerance.
\end{assumption}

Define
\[
  b_{\rm sens}(m,T)
  =
  \norm{A_mX_T-X_T}_{L^2(\Pp;\Hsp)},
  \qquad
  b_{\rm end}(T)=\Wtwo(\mu_T,\mu_1).
\]

\begin{theorem}[End-to-end error decomposition]
\label{thm:end-to-end}
Suppose Assumptions~\ref{ass:learned-lipschitz} and \ref{ass:solver} hold
and, with probability at least $1-\delta$, the learned field has the
excess-risk bound
\[
  \cL_m(\widehat v_m)-\cL_m(v_m^\star)
  \le\overline{\mathfrak R}_{m,p,n}(\delta).
\]
Then, with the same probability,
\begin{align}
  \Wtwo(\widetilde\mu_{m,T},\mu_1)
  \le{}&
  b_{\rm end}(T)
  +b_{\rm sens}(m,T)
  +\varepsilon_{{\rm sol},m}(h)
  \nonumber\\
  &+
  T e^{\widehat\Lambda}
  \overline{\mathfrak R}_{m,p,n}(\delta)^{1/2}.
  \label{eq:end-to-end}
\end{align}
Under Assumption~\ref{ass:parametric-class} and the approximate ERM
condition \eqref{eq:approximate-erm}, Theorem~\ref{thm:fast-oracle} permits
\begin{equation}
\label{eq:end-to-end-fast-risk}
  \overline{\mathfrak R}_{m,p,n}(\delta)
  =
  2a_{m,p}
  +\frac32\varepsilon_{\rm opt}
  +\frac{10BGR}{n}
  +\frac{80B^2q_{p,n,\delta}}{n}.
\end{equation}
In the realizable case, use \eqref{eq:fast-oracle-realizable}.
\end{theorem}

\begin{proof}
Apply the triangle inequality through $\widehat\mu_{m,T}$, $\mu_T^m$,
and $\mu_T$. The solver bound,
Theorem~\ref{thm:learned-law-stability}, and the canonical reconstruction
coupling bound the respective terms. The risk choices follow from
Theorem~\ref{thm:fast-oracle}.
\end{proof}

\begin{corollary}[End-to-end bound for literal point sensors]
\label{cor:point-end-to-end}
Suppose the regularity-space setting of
Section~\ref{subsec:point-sensor-setting} holds, with bounded linear sensors
and reconstructions satisfying \eqref{eq:regularity-reconstruction} and
$\sigma(\tau,S_mX_\tau)=\sigma(\tau,A_mX_\tau)$.  If the learned-field, solver, and
statistical assumptions of Theorem~\ref{thm:end-to-end} hold in the inherited
$\Hsp$ norm on $\Hm$, then \eqref{eq:end-to-end} remains valid with
\[
  b_{\rm sens}(m,T)
  =
  \norm{R_mS_mX_T-X_T}_{L^2(\Pp;\Hsp)}.
\]
Furthermore,
$\Ee\sup_{t\le T}\norm{R_mS_mX_t-X_t}_{\Hsp}^2\to0$.
\end{corollary}

\begin{proof}
The path $A_mX$ is absolutely continuous in $\Hm$ with derivative
$A_mU\in L^2(\Hsp)$, so the finite superposition, learning, and
coupling arguments apply in the inherited norm. Almost every $X$-path
has compact image in $\Vsp$; uniform boundedness and pointwise
consistency make reconstruction uniform on that image. Domination by a
constant times $\sup_t\norm{X_t}_{\Vsp}^2$ gives the mean-square limit.
\end{proof}

\begin{remark}
To extend Theorem~\ref{thm:population-flow-convergence} to point sensors,
the continuum flow must preserve $\Vsp$, with the corresponding moment
bounds for its trajectories. Corollary~\ref{cor:point-end-to-end} uses
only the regularity of the interpolation paths.
\end{remark}

A globally Lipschitz flow is invertible and need not reach a
lower-dimensional endpoint from a full-support source. Stopping at $T<1$
retains the bias $b_{\rm end}(T)$; for \eqref{eq:linear-interpolation},
\begin{equation}
\label{eq:endpoint-linear}
  b_{\rm end}(T)
  \le
  \norm{X_T-X_1}_{L^2}
  =
  (1-T)\norm{X_1-X_0}_{L^2}.
\end{equation}
On a common probability space, let $m_k,n_k\to\infty$, $h_k\to0$,
and possibly $T_k\uparrow1$. If
\[
  b_{\rm end}(T_k)\to0,\quad
  b_{\rm sens}(m_k,T_k)\to0,\quad
  \varepsilon_{{\rm sol},m_k}(h_k)\to0,\quad
  \overline{\mathfrak R}_{m_k,p_k,n_k}(\delta_k)\to0,
\]
while $\widehat\Lambda_k$ remains bounded, then
$\Wtwo(\widetilde\mu_{m_k,T_k},\mu_1)\to0$ in probability when
$\delta_k\to0$, and almost surely when $\sum_k\delta_k<\infty$, by
Theorem~\ref{thm:end-to-end} and Borel--Cantelli.

\begin{remark}[Joint model--sample growth]
\label{rem:joint-model-sample-growth}
Using \eqref{eq:end-to-end-fast-risk} requires
$a_{m_k,p_k},\varepsilon_{{\rm opt},k}\to0$ and
\begin{equation}
\label{eq:joint-model-sample-growth}
  \frac{B_kG_kR_k}{n_k}\longrightarrow0,
  \qquad
  B_k^2
    \frac{
      p_k\log(1+2n_k)+\log(2/\delta_k)
    }{n_k}
  \longrightarrow0.
\end{equation}
At fixed depth, bounded $R_k$, and under the output/data controls stated
after Theorem~\ref{thm:qno-uniform-constants}, if
$n_k,p_k,\delta_k^{-1}$ are polynomially related, the bounded-activation
estimates $B_k=O(1)$, $G_k=O(p_k^{L/2})$ make
\eqref{eq:joint-model-sample-growth} follow from
$p_k^{\max\{L/2,1\}}\log p_k=o(n_k)$.  For the raw magnitude recurrence,
$B_k,G_k=O(p_k^{L/2})$, and $p_k^{L+1}\log p_k=o(n_k)$ suffices; normalized
dictionaries with $B_k,G_k=O(1)$ leave only the entropy condition.
The observation level must also grow fast enough that the realization
bound \eqref{eq:qno-realization-bound}, evaluated at $(m_k,p_k)$, tends
to zero. These conditions specify how sample size and observation
resolution must scale with model size.
\end{remark}

The approximation decomposition \eqref{eq:approximation-decomposition-bound}
further separates
\begin{equation}
\label{eq:resolved-approximation}
 \sqrt{a_{m,p}}\le
 \underbrace{d_m}_{\text{operator realization}}
 +\underbrace{\sqrt{a_p^\infty}}_{\text{continuum approximation}}
 +\underbrace{e_m}_{\text{target refinement}}.
\end{equation}
Table~\ref{tab:end-to-end-map} lists the corresponding controls.

\begin{table}[t]
\centering
\caption{Error terms and the quantities that control them.}
\label{tab:end-to-end-map}
\small
\begin{tabular}{@{}>{\raggedright\arraybackslash}p{.17\linewidth}
                    >{\raggedright\arraybackslash}p{.31\linewidth}
                    >{\raggedright\arraybackslash}p{.43\linewidth}@{}}
\toprule
Term & Bound or assumption & Primary control \\
\midrule
$b_{\rm end}(T)$ & endpoint coupling; \eqref{eq:endpoint-linear}
  & stopping time and interpolation \\
$b_{\rm sens}(m,T)$ & canonical reconstruction coupling
  & sensor geometry, reconstruction, and $m$ \\
$d_m$ & Proposition~\ref{prop:qno-realization}
  & input regularity, cubature, and output reconstruction \\
$\sqrt{a_p^\infty}$ & continuum approximation error
  & architecture and model dimension $p$ \\
$e_m$ & Theorem~\ref{thm:target-consistency} and its corollaries
  & reconstruction tails and observation information \\
$(BGR+B^2q)/n$ & Theorem~\ref{thm:fast-oracle}
  & $n,p$ and the envelopes $B,G,R$ \\
$\varepsilon_{\rm opt}$ & approximate ERM condition
  & optimization accuracy \\
$\varepsilon_{{\rm sol},m}(h)$ & Assumption~\ref{ass:solver}
  & solver, tolerance, and mesh-uniform constants \\
$e^{\widehat\Lambda}$ & Theorem~\ref{thm:learned-law-stability}
  & learned-field Lipschitz control \\
\bottomrule
\end{tabular}
\end{table}

\section{Gaussian tests}
\label{sec:gaussian}

Let $\Hsp=\ell^2(\N)$ with basis $(e_k)$ and
$Ce_k=k^{-2\alpha}e_k=:c_ke_k$, $\alpha>1/2$. Thus $C$ is positive,
injective, and trace class.

\subsection{A commuting baseline}

Take independent
$X_0\sim\mathcal N(0,C)$ and
$X_1\sim\mathcal N(0,\rho^2C)$, where $\rho>0$, and define
\[
  X_t=(1-t)X_0+tX_1,\quad U=X_1-X_0,\quad
  q(t)=(1-t)^2+\rho^2t^2.
\]
Since $\operatorname{Cov}(U,X_t)=\{\rho^2t-(1-t)\}C$ and
$\operatorname{Cov}(X_t)=q(t)C$, joint Gaussianity makes $U-a(t)X_t$
independent of $X_t$, where
$a(t)=\{\rho^2t-(1-t)\}/q(t)=q'(t)/(2q(t))$. Hence
\begin{equation}
\label{eq:gaussian-velocity}
  v^\star(t,X_t)=\Ee[U\mid X_t]=a(t)X_t,
  \qquad
  \Phi_t(x)=\sqrt{q(t)}\,x.
\end{equation}
Thus $(\Phi_t)_\#\mathcal N(0,C)=\mu_t$ and
$\Ee[P_mU\mid P_mX_t]=a(t)P_mX_t$. For projection onto the first $m$
coordinates,
\begin{equation}
\label{eq:gaussian-projection-exact}
 \begin{aligned}
  \Wtwo^2((P_m)_\#\mu_t,\mu_t)&=q(t)\sum_{k>m}k^{-2\alpha},\\
  \Wtwo((P_m)_\#\mu_t,\mu_t)&=
  \sqrt{\frac{q(t)}{2\alpha-1}}\,
  m^{1/2-\alpha}\{1+O(m^{-1})\}.
 \end{aligned}
\end{equation}
Indeed, orthogonality makes $(X_t,P_mX_t)$ an optimal coupling and the
integral-test expansion of the spectral tail gives the final equality.

\subsection{A noncommuting observation model}
\label{subsec:noncommuting-gaussian}

Write $e_k^+=e_{2k-1}$ and $e_k^-=e_{2k}$.  Fix $r_+>r_->0$ and define
\[
  C_0e_k^\pm=c_ke_k^\pm,\qquad
  C_1e_k^+=r_+c_ke_k^+,\qquad
  C_1e_k^-=r_-c_ke_k^-.
\]
For independent $X_i\sim\mathcal N(0,C_i)$, set
$X_t=(1-t)X_0+tX_1$ and
\[
  Q_\pm=(1-t)^2+r_\pm t^2,\qquad
  D_\pm=r_\pm t-(1-t),\qquad
  k_\pm=D_\pm/Q_\pm.
\]
The continuum target is $K_tx$, where $K_te_k^\pm=k_\pm e_k^\pm$.
At level $m$, observe every earlier pair and only
$g_m=(e_m^++e_m^-)/\sqrt2$ in the boundary pair:
\[
  \mathsf G_m
  =
  \operatorname{span}\{e_k^+,e_k^-:k<m\}
  \oplus\operatorname{span}\{g_m\},
\]
with orthogonal projection $\Pi_m$.  Although $\Pi_m\to I$ strongly, it does
not commute with $\operatorname{Cov}(X_t)$ when $Q_+\ne Q_-$.

\begin{proposition}[Exact noncommuting target]
\label{prop:noncommuting-gaussian}
For $z\in\mathsf G_m$,
\begin{equation}
\label{eq:noncommuting-finite-target}
  v_m^\star(t,z)
  =
  \sum_{k<m}
  \{k_+z_k^+e_k^++k_-z_k^-e_k^-\}
  +
  \kappa(t)\inner{z}{g_m}g_m,
  \qquad
  \kappa(t)=\frac{D_++D_-}{Q_++Q_-}.
\end{equation}
The projected restriction has boundary coefficient
$\overline k=(k_++k_-)/2$, which generally differs from $\kappa$.
Nevertheless,
\begin{equation}
\label{eq:noncommuting-uniform-lipschitz}
  \Lip(v_m^\star(t,\cdot))
  \le L(t):=\max\{\abs{k_+(t)},\abs{k_-(t)}\}
\end{equation}
uniformly in $m$.

Let $\mathsf T_m=\sum_{k>m}c_k$ and
$A(t)=D_+^2/Q_++D_-^2/Q_-$.  Then
\begin{align}
\label{eq:noncommuting-target-error}
  E_m(t)
  &:=
  \Ee\norm{K_tX_t-v_m^\star(t,\Pi_mX_t)}^2\\
  &=
  \mathsf T_mA(t)
  +
  c_m\left[
    A(t)-\frac{(D_++D_-)^2}{2(Q_++Q_-)}
  \right],\nonumber\\
\label{eq:noncommuting-target-bound}
  E_m(t)
  &\le
  B_m(t):=
  \mathsf T_mA(t)+\frac{c_m}{2}A(t)
  +
  L(t)^2\left(\mathsf T_m+\frac{c_m}{2}\right)(Q_++Q_-).
\end{align}
For every fixed $T>0$, both
$\int_0^T E_m(t)\dt$ and $\int_0^T B_m(t)\dt$ are
$\Theta(m^{1-2\alpha})$.
\end{proposition}

\begin{proof}
Complete pairs are observed exactly. For the boundary pair let
$s=(\xi_++\xi_-)/\sqrt2$ and
$R=(k_+\xi_++k_-\xi_-)/\sqrt2$, where
$\operatorname{Var}(\xi_\pm)=c_mQ_\pm$. Since
$\operatorname{Cov}(R,s)=c_m(D_++D_-)/2$ and
$\operatorname{Var}(s)=c_m(Q_++Q_-)/2$, Gaussian conditioning gives
\eqref{eq:noncommuting-finite-target}.
The coefficient $\kappa=(Q_+k_++Q_-k_-)/(Q_++Q_-)$ is a convex
combination, proving \eqref{eq:noncommuting-uniform-lipschitz}.

The unobserved boundary component, conditional residual
$R-\Ee[R\mid s]$, and later pairs form an orthogonal error decomposition
with variances
\[
 \frac{c_mA}{2},\qquad
 \frac{c_m}{2}\left[A-\frac{(D_++D_-)^2}{Q_++Q_-}\right],\qquad
 \mathsf T_mA,
\]
whose sum is \eqref{eq:noncommuting-target-error}. The projection tails are
$\Ee\norm{(I-\Pi_m)K_tX_t}^2=(\mathsf T_m+c_m/2)A$ and
$\Ee\norm{(I-\Pi_m)X_t}^2=(\mathsf T_m+c_m/2)(Q_++Q_-)$.
Corollary~\ref{cor:orthogonal-target} gives
\eqref{eq:noncommuting-target-bound}; the rates follow from
$\mathsf T_m\sim m^{1-2\alpha}/(2\alpha-1)$ and
$c_m=O(m^{-2\alpha})$.
\end{proof}

At $\alpha=3/2$, $r_+=4$, $r_-=1/4$, and $t=1/2$, the multipliers are
$k_+=1.2$, $k_-=-1.2$, $\overline k=0$, and $\kappa=0.72$.
For a finite-sample check, we drew $n$ independent boundary pairs with
$\xi_+\sim\mathcal N(0,1.25)$, $\xi_-\sim\mathcal N(0,0.3125)$ and fitted
\[
 \widehat\kappa_n=\frac{\sum_{i=1}^ns_iR_i}{\sum_{i=1}^ns_i^2},
 \qquad s_i=\frac{\xi_{i,+}+\xi_{i,-}}{\sqrt2},
 \qquad R_i=\frac{1.2(\xi_{i,+}-\xi_{i,-})}{\sqrt2}.
\]
Table~\ref{tab:gaussian-sanity-check} reports 200 repetitions using
PCG64 seeds $0,\ldots,199$. The estimate approaches $0.72$, while
projected restriction remains zero.

\begin{table}[t]
\centering
\caption{Mean, sample standard deviation, and root-mean-square error (RMSE) of
$\widehat\kappa_n$ relative to $0.72$ over 200 repetitions.}
\label{tab:gaussian-sanity-check}
\small
\begin{tabular}{rrrr}
\toprule
$n$ & Mean & Sample SD & RMSE \\
\midrule
128&0.7061&0.0887&0.0896\\
512&0.7221&0.0425&0.0424\\
2048&0.7182&0.0214&0.0215\\
8192&0.7193&0.0099&0.0099\\
\bottomrule
\end{tabular}
\end{table}

\subsection{Clipping and a bounded learning regime}
\label{subsec:gaussian-clipping}

For $r>0$, let $T_r$ be metric projection onto the closed radius-$r$ ball.

\begin{proposition}[Cost of radial Gaussian clipping]
\label{prop:gaussian-clipping}
Let $(G_0,G_1)$ be any coupling of centered Gaussian elements with nonzero
trace-class covariances $C_0,C_1$.  Define
\[
  G_i^{r_i}=T_{r_i}(G_i),\quad
  X_t=(1-t)G_0+tG_1,\quad
  X_t^{\boldsymbol r}=(1-t)G_0^{r_0}+tG_1^{r_1},
\]
and
$\epsilon_i(r_i)=\{\Ee(\norm{G_i}-r_i)_+^2\}^{1/2}$.
Then
\begin{align}
\label{eq:clipping-w2}
  \norm{X_t^{\boldsymbol r}}
  &\le(1-t)r_0+tr_1,\qquad
  \Wtwo(\Law(X_t^{\boldsymbol r}),\Law(X_t))
  \le(1-t)\epsilon_0+t\epsilon_1,\\
\label{eq:gaussian-clipping-tail}
  \epsilon_i(r_i)^2
  &\le
  2\norm{C_i}_{\rm op}
  \exp\!\left[
    -\frac{(r_i-\sqrt{\operatorname{Tr}C_i})^2}
           {2\norm{C_i}_{\rm op}}
  \right]
\end{align}
whenever $r_i>\sqrt{\operatorname{Tr}C_i}$.  If
$\sup_m\norm{A_m}_{\rm op}\le C_A$, then
\begin{equation}
\label{eq:clipped-target-bound}
  \sup_m\norm{A_m(G_1^{r_1}-G_0^{r_0})}
  \le C_A(r_0+r_1)
  \quad\text{almost surely}.
\end{equation}
\end{proposition}

\begin{proof}
Metric projection gives
$\norm{G_i-T_{r_i}(G_i)}=(\norm{G_i}-r_i)_+$; coupling and Minkowski
prove \eqref{eq:clipping-w2}. Integrate Gaussian concentration
$\Pp(\norm{G_i}>\sqrt{\operatorname{Tr}C_i}+s)
\le e^{-s^2/(2\norm{C_i}_{\rm op})}$ \cite{bogachev1998gaussian}
to obtain \eqref{eq:gaussian-clipping-tail}. The last bound follows from
$\norm{A_m}\le C_A$ and $\norm{G_i^{r_i}}\le r_i$.
\end{proof}

\begin{proposition}[A simultaneous bounded regime]
\label{prop:simultaneous-bounded-regime}
Let $G\sim\mathcal N(0,C)$, $G^r=T_r(G)$, $\rho>0$, and
\[
  X_t^r=s_\rho(t)G^r,\quad U^r=(\rho-1)G^r,\quad
  s_\rho(t)=1+t(\rho-1).
\]
For any strongly convergent finite-rank orthogonal projections $P_m$,
\[
  v^\star(t,x)=\ell_\rho(t)x,\qquad
  v_m^\star(t,z)=\ell_\rho(t)z,\qquad
  \ell_\rho(t)=\frac{\rho-1}{s_\rho(t)},
\]
and
$\sup_m\int_0^1\Lip(v_m^\star(t,\cdot))\dt=\abs{\log\rho}$.
For $R\ge1$, the class
$f_{\theta,m}(t,z)=\theta\ell_\rho(t)z$, $\abs\theta\le R$,
has $a_{m,1}=0$, learned-field Lipschitz constant
$R\ell_{\max}$, and satisfies Assumption~\ref{ass:parametric-class} with
\[
  B=\max\{\abs{\rho-1}r,R\ell_{\max}s_{\max}r\},
  \qquad
  G=\ell_{\max}s_{\max}r,
\]
where
$s_{\max}=\max\{1,\rho\}$ and
$\ell_{\max}=\abs{\rho-1}/\min\{1,\rho\}$.
Moreover,
\[
  \Wtwo(\Law(\rho G^r),\mathcal N(0,\rho^2C))
  \le\rho\{\Ee(\norm G-r)_+^2\}^{1/2}.
\]
\end{proposition}

\begin{proof}
Because $s_\rho(t)>0$, $G^r=X_t^r/s_\rho(t)$, and the same identity holds
after applying $P_m$; this proves the two conditional targets and the
integrated Lipschitz constant.  On the training support,
$\norm{P_mX_t^r}\le s_{\max}r$ and
$\norm{P_mU^r}\le\abs{\rho-1}r$, which give the stated envelope and
parameter-Lipschitz constants.  The class contains the target at $\theta=1$.
The final inequality couples $\rho G^r$ with $\rho G$.
\end{proof}

\begin{corollary}[Explicit clipped-Gaussian rate]
\label{cor:explicit-clipped-gaussian-rate}
Let $Ce_k=k^{-2\alpha}e_k$ with $\alpha>1/2$, let $P_m$ project onto the
first $m$ coordinates, and use the clipped scaling model and class of
Proposition~\ref{prop:simultaneous-bounded-regime}, with $R\ge1$ and
$r>\sqrt{\operatorname{Tr}C}$.  Let $\widetilde\mu_{m,1}^{r}$ be the law
returned by a solver satisfying Assumption~\ref{ass:solver}, and write its
error bound as $\varepsilon_{{\rm sol},m}^{(r)}(h)$.  Set
\[
 b_\rho=\max\{\abs{\rho-1},R\ell_{\max}s_{\max}\},\quad
 g_\rho=\ell_{\max}s_{\max},
\]
\[
 \Psi_{r,n}(\delta)
 =\frac32\varepsilon_{\rm opt}
 +\frac{10b_\rho g_\rho Rr^2}{n}
 +\frac{40b_\rho^2r^2}{n}
   \{\log(1+2n)+\log(2/\delta)\}.
\]
If $\widehat v_m$ satisfies \eqref{eq:approximate-erm}, then with probability
at least $1-\delta$,
\begin{align}
 \Wtwo\!\left(\widetilde\mu_{m,1}^{r},\mathcal N(0,\rho^2C)\right)
 \le{}&\rho\epsilon(r)
 +\rho\left(\sum_{k>m}k^{-2\alpha}\right)^{1/2}
 +\varepsilon_{{\rm sol},m}^{(r)}(h)
 +e^{R\abs{\log\rho}}\Psi_{r,n}(\delta)^{1/2},
 \label{eq:explicit-clipped-gaussian-bound}
\end{align}
where
\[
 \epsilon(r):=\{\Ee(\norm G-r)_+^2\}^{1/2}
 \le\sqrt{2\norm C_{\rm op}}
 \exp\!\left[-\frac{(r-\sqrt{\operatorname{Tr}C})^2}
                    {4\norm C_{\rm op}}\right].
\]
If $\varepsilon_{{\rm sol},m}^{(r)}(h)\le C_{\rm sol}r h^\beta$ with
$C_{\rm sol}$ independent
of $m$ and $r$, choose, for $n\ge2$,
\[
 \begin{gathered}
 r_n=\sqrt{\operatorname{Tr}C}+\sqrt{2\norm C_{\rm op}\log n},\qquad
 m_n=\left\lceil n^{1/(2\alpha-1)}\right\rceil,\\
 p=T=1,\quad h_n=n^{-1/(2\beta)},\quad \delta_n=n^{-2},\quad
 \varepsilon_{{\rm opt},n}=O((\log n)^2/n).
 \end{gathered}
\]
Then
\begin{equation}
\label{eq:explicit-clipped-gaussian-rate}
 \Wtwo\!\left(\widetilde\mu_{m_n,1}^{r_n},
                    \mathcal N(0,\rho^2C)\right)
 =O((\log n)/\sqrt n)
\end{equation}
with probability at least $1-n^{-2}$, and eventually almost surely on a
common probability space.
\end{corollary}

\begin{proof}
Proposition~\ref{prop:simultaneous-bounded-regime} gives
$a_{m,1}=0$, $B=b_\rho r$, $G=g_\rho r$, and
$\widehat\Lambda\le R\abs{\log\rho}$.  Since $G^r$ is a scalar contraction,
$\norm{(I-P_m)\rho G^r}_{L^2}
\le\rho(\sum_{k>m}k^{-2\alpha})^{1/2}$.
Apply Theorem~\ref{thm:end-to-end} at $T=1$ with
\eqref{eq:fast-oracle-realizable}, then compare $\rho G^r$ with $\rho G$.
For the schedule above, clipping and sensing are $O(n^{-1/2})$ and the
solver error is $O(\sqrt{\log n}/\sqrt n)$, while
$r_n^2=O(\log n)$ makes the statistical term
$O((\log n)/\sqrt n)$.  Borel--Cantelli gives the last assertion.
\end{proof}

\begin{remark}
The scaling coupling $G_1^r=\rho G_0^r$ makes the target exactly
representable, so the generic statistical bound is conservative here.
For independent endpoints, clipping changes the conditional drift because
the endpoint pair is no longer jointly Gaussian. Applying the bound in
that setting requires estimates for the new drift.
\end{remark}
\section{Discussion}
\label{sec:discussion}

Refining finite observations changes both the representation of a random
function and the conditional velocity used to train its generator.
Theorem~\ref{thm:target-consistency} proves that these velocities converge
in mean square for consistent reconstructions, even when successive
observations are not nested. The quadrature-operator estimates then connect
the continuum model to its implementation, and
Theorem~\ref{thm:end-to-end} bounds the error in the sampled distribution.

The next step is to estimate the conditional drift for independent
endpoints after clipping. That would extend the explicit Gaussian rate
beyond the scaling coupling. Moment-based statistical bounds would allow
untruncated data. For architectures such as FNO, the corresponding task is
to bound the Fourier layers and their discretization error.
Table~\ref{tab:gaussian-sanity-check} isolates the effect of conditioning;
training a full operator would also measure approximation and optimization
errors.

These results concern distributions on function space. Evaluating
infinite-dimensional likelihoods requires quasi-invariance analysis
\cite{bogachev1998gaussian}, and recovering frequencies absent from both
the data and model requires additional assumptions.

\bibliographystyle{siamplain}
\bibliography{references}

\section*{Declarations}
The author used ChatGPT (OpenAI) to assist with manuscript editing and
LaTeX formatting. The author assumes responsibility for all content.

\end{document}